\pdfoutput=1
\PassOptionsToPackage{table}{xcolor}
\documentclass{article}
\usepackage{iclr2027_conference,times}

\iclrfinalcopy

\usepackage{hyperref}
\usepackage{multirow}
\usepackage{float}
\usepackage{placeins}
\usepackage[table]{xcolor}
\usepackage{graphicx}
\usepackage{caption}
\usepackage{subcaption}
\usepackage{placeins}

\usepackage{booktabs}
\usepackage{amsmath,amsthm,amsfonts,amssymb,mathtools,bm}
\usepackage{colortbl}
\usepackage{dcolumn}
\usepackage{algorithm}
\usepackage{algpseudocode}
\usepackage{tikz}
\usepackage{comment}

\newcommand{\email}[1]{\texttt{#1}}

\usepackage{amsmath,amsfonts,bm}

\def\1{\bm{1}}

\def\mI{{\bm{I}}}

\def\mU{{\bm{U}}}

\DeclareMathAlphabet{\mathsfit}{\encodingdefault}{\sfdefault}{m}{sl}
\SetMathAlphabet{\mathsfit}{bold}{\encodingdefault}{\sfdefault}{bx}{n}

\makeatletter
\DeclareRobustCommand\onedot{\futurelet\@let@token\@onedot}
\def\@onedot{\ifx\@let@token.\else.\null\fi}

\makeatother

\graphicspath{{figures/}}

\newcolumntype{d}{D{.}{.}{3.2}}
\makeatletter
\newcolumntype{B}{>{\boldmath\DC@{.}{.}{3.2}}c<{\DC@end}}
\makeatother

\newcommand{\tbox}[2][c]{\begin{tabular}[c]{@{}#1@{}}#2\end{tabular}}

\newcommand{\thead}[1]{\multicolumn{1}{c}{\textbf{#1}}}

\newcommand{\shadowkv}{\textsc{ShadowKV}}

\DeclarePairedDelimiter{\parens}{\lparen}{\rparen}

\DeclarePairedDelimiter{\norm}{\|}{\|}

\newcommand{\hpad}{\hspace*{\fill}}
\newcommand{\subplotsizeslack}{1}
\NewDocumentCommand{\subplot}{ O{1} O{} m m m }{%
    \hpad%
    \begin{subfigure}{(\textwidth / #1) * \subplotsizeslack}
        \includegraphics[width=\textwidth, #2]{#3}
        \caption{#4}\label{#5}
    \end{subfigure}%
}

\usepackage{cleveref}

\newtheorem{proposition}{Proposition}

\newcommand{\methodverb}{{\emph{iS}-KV}}
\newcommand{\method}{\methodverb}
\newcommand{\hto}{\textsc{H2O}}
\newcommand{\snapkv}{\textsc{SnapKV}}
\newcommand{\streaming}{\textsc{StreamingLLM}}
\newcommand{\ruler}{\textsc{RULER}}

\newcommand{\drift}{\ensuremath{\delta}}

\title{\textsc{\method{}}: Online Low-Rank KV Cache\\
Compression via Block-Incremental SVD}

\author{%
Yiren Zhao\textsuperscript{1}\quad
Guanghui Song\textsuperscript{1}\quad
Tianrui Qin\textsuperscript{2}\\
\bfseries Kejiang Ye\textsuperscript{3}\quad
Cheng-zhong Xu\textsuperscript{4}\quad
Xitong Gao\textsuperscript{3,5}\thanks{Corresponding author: \email{xt.gao@siat.ac.cn}.}\\[4pt]
\normalfont
\textsuperscript{1}The Hong Kong University of Science and Technology (Guangzhou)\\
\textsuperscript{2}Guangdong OPPO Mobile Telecommunications Corp., Ltd.\\
\textsuperscript{3}Shenzhen Institutes of Advanced Technology, Chinese Academy of Sciences\\
\textsuperscript{4}University of Macau\\
\textsuperscript{5}Shenzhen University of Advanced Technology
}

\begin{document}

\raggedbottom

\maketitle
\fancyhead{}
\renewcommand{\headrulewidth}{0pt}

\begin{abstract}
Long chain-of-thought reasoning substantially increases KV-cache memory during autoregressive decoding, as every generated token introduces new key and value states and causes the cache to grow linearly with decoding length.
Existing KV-cache compression methods typically control this growth through token eviction, but irreversible deletion can remove historical states that later reasoning may need to revisit.
SVD-based low-rank compression provides an alternative by retaining all positions with a more compact representation.
However, extending it from a fixed prompt cache to online decoding is non-trivial.
Through our investigation, we find that if the basis is updated for new tokens while old tokens keep their coordinates in the old basis, the stored history drifts substantially.
Based on this observation, we propose \textbf{\method{}}, an online low-rank KV-cache compression method for long-horizon reasoning.
\method{} keeps a recent window exact while incrementally folding older states into bounded-rank representations.
As the low-rank basis evolves, it synchronizes historical coordinates with the updated basis to maintain representation consistency.
On DeepSeek-R1-Distill-Llama-8B, \method{} achieves 82.6\% accuracy at 4.06\(\times\) persistent-KV compression, close to the original model's 83.6\%.
On Qwen3-8B, it achieves 89.2\% accuracy at 5.64\(\times\) compression.
Under matched memory budgets, \method{} \mbox{consistently outperforms token-eviction baselines.}
\end{abstract}

\section{Introduction}\label{sec:intro}

Long chain-of-thought (CoT) reasoning increases the memory needed during generation.
Each generated token adds key and value states to the KV cache, so a long
reasoning trace can require substantially more KV cache storage than its prompt
\citep{guo2025deepseek,yang2025qwen3technicalreport}.
The challenge is to reduce the KV storage cost during the decoding-phase
without losing information that later reasoning may need.

Two observations motivate our design. First, old tokens can become important
again. A token that receives little attention at the current step may receive
strong attention thousands of tokens later. Once evicted, it
cannot be recovered. On Qwen3-8B, for example, R-KV reduces
MATH-500 accuracy from 94.4\% to 68.8\% at about \(5.6\times\) compression.
Second, KV states have strong low-rank structure, so old tokens
can be compressed instead of deleted. Doing this online requires
care: when the low-rank basis changes as new tokens arrive, the coordinates of
previously compressed tokens must change with it. Otherwise, the stored
history no longer represents the same states.

Token-eviction methods reduce cache size by keeping selected positions and
removing the rest, typically using attention scores, recency, or redundancy
\citep{zhang2023h2oheavyhitteroracleefficient,xiao2024efficientstreaminglanguagemodels,cai2026r}.
These signals describe which tokens appear important under the queries seen so
far, but they cannot predict all future queries. \Cref{fig:motivation}(a)
illustrates this issue on long reasoning traces from
DeepSeek-R1-Distill-Llama-8B and Qwen3-8B on MATH-500, AIME 2024, and AIME
2025. At decoding step 1024, we mark the 75\% of tokens with the lowest recent
attention, without removing them. As decoding continues, many of these tokens
become important again. On AIME, about 30\% receive strong attention again
within the next 8K tokens, and this fraction grows to about 60\% by 30K
tokens. At matched memory, actual eviction by R-KV loses 5.4--11.0 points on
DeepSeek-R1-Distill-Llama-8B and 23.0--33.6 on Qwen3-8B (\Cref{tab:main-results}).
\begin{figure*}[!t]
  \centering
  \includegraphics[width=\textwidth]{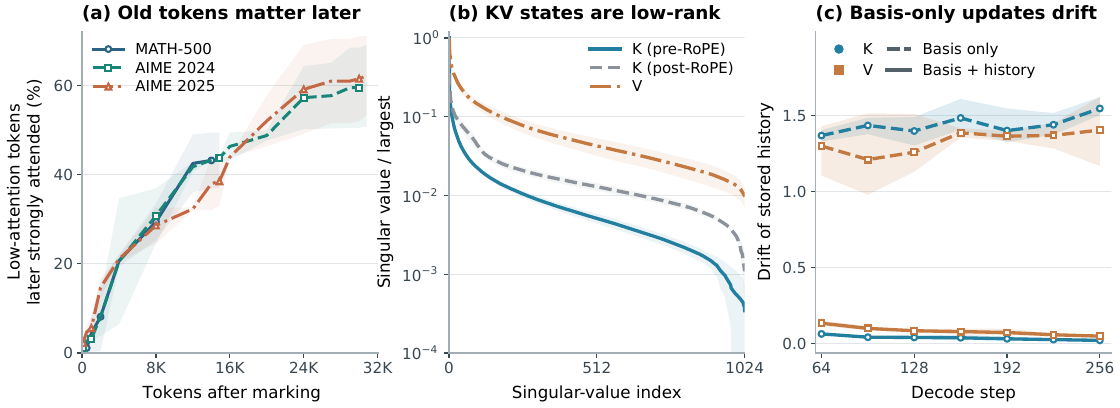}
  \caption{\textbf{Keep old positions, compress their KV states, and update them together.}
  \emph{(a)} At step 1024, we mark the 75\% least-attended tokens (none are
  removed) and plot the share that later receives strong attention.
  \emph{(b)} Singular values of generated K/V states, divided by the largest.
  \emph{(c)} Drift of the stored history caused by each update. Dashed: the
  basis is updated but old coordinates are kept. Solid: the basis and old
  coordinates are updated jointly, which equals the \method{} update in exact
  arithmetic. Bands show observed ranges; \Cref{app:fig1} gives all settings.}
  \label{fig:motivation}
\end{figure*}

Low-rank compression offers a different approach: keep every token position,
but store each old state more compactly. Each position is represented by a
short coefficient vector over a basis shared across positions.
\Cref{fig:motivation}(b) shows that generated KV states are indeed low-rank.
Keys have the fastest singular-value decay before rotary positional embeddings
(RoPE), consistent with prior work \citep{sun2025shadowkvkvcacheshadows},
while values decay more slowly and are harder to compress. The main difficulty appears during online decoding. As new states arrive, the
low-rank basis must adapt. But the coefficients of old states were computed in
the previous basis. If the basis is updated while the old coefficients remain
unchanged, the reconstructed historical states change as well.
In other words, the same coefficients represent different vectors under a new basis.
\Cref{fig:motivation}(c) measures this drift. At step 256, updating only the
basis gives a drift of 1.55 for keys and 1.40 for values. Updating the old
coefficients together with the basis reduces these values to 0.023 and 0.051.
This result shows that adapting the basis and preserving old states must be
handled together.

We therefore propose \method{}, an online low-rank KV-cache compression method
designed for long-horizon reasoning. \method{} keeps recent states at their
original precision and folds older states into separate low-rank K
and V representations. Using block incremental SVD \citep{Brand2006may}, it
updates both the low-rank basis and the coefficients of all previously
compressed positions. Keys are compressed before RoPE, where their low-rank structure is stronger.
During attention, the reconstructed historical states are combined with the
recent exact states. Every token position thus remains available to
future queries, although older states are approximate. A fixed
rank reduces the storage per old token, so \method{} slows cache growth rather
than making it constant in size.

On MATH-500, \method{} achieves 82.6\% accuracy at \(4.06\times\)
persistent-KV compression on DeepSeek-R1-Distill-Llama-8B, close to the
original model's 83.6\%. On Qwen3-8B, it achieves 89.2\% at \(5.64\times\)
compression, compared with 68.8\% for matched-memory R-KV. It
outperforms both R-KV and SnapKV at all six matched-memory operating points and
remains within 1.0--6.6 accuracy points of the original model. On AIME 2024 at
rank 192, \method{} on Qwen3-8B matches the original model with 23 of 30
correct answers, compared with 13 for matched-memory R-KV. Memory
matching counts low-rank factors, exact entries, and R-KV's staging buffer at
the original model's output lengths. This
avoids counting shorter generations as compression gains.
\mbox{Our contributions are:}
\begin{itemize}
  \item \textbf{Online K/V compression.} We develop \method{} to retain every
  token position, keeping recent states exact while storing older states in
  bounded-rank form.
  \item \textbf{Consistent history updates.} We show that a basis-only update
  makes the stored history drift, and update the basis and old
  coefficients jointly via block incremental SVD.
  \item \textbf{Reasoning accuracy at matched memory.} On two 8B models,
  \method{} outperforms R-KV and SnapKV at all six matched MATH-500 operating
  points. AIME, RULER, ablation, and latency results further show its
  strengths and limitations.
\end{itemize}

\section{Related Work}\label{sec:related}

\paragraph{Token selection and eviction.}
Token-selection methods reduce KV cache storage requirement by keeping a subset of cached positions.
\hto{} retains heavy hitters and recent tokens \citep{zhang2023h2oheavyhitteroracleefficient}, while \streaming{} combines attention sinks with a sliding window \citep{xiao2024efficientstreaminglanguagemodels}.
FastGen selects head-specific policies through attention profiling \citep{ge2024model}.
For long prompts, \snapkv{} selects positions using a trailing observation window \citep{li2024snapkvllmknowslooking}, and PyramidKV allocates different cache budgets across layers \citep{cai2024pyramidkv}.
R-KV targets long reasoning with redundancy-aware selection \citep{cai2026r}.
These methods preserve selected states at the expense of historical coverage.
\method{} instead keeps every position in exact or low-rank form, trading representation precision rather than token availability.

\paragraph{Attention efficiency and cache quantization.}
Other approaches reduce different parts of the memory cost.
FlashAttention reduces attention IO \citep{dao2022flashattention}; PagedAttention reduces cache fragmentation and enables sharing \citep{kwon2023efficientmemorymanagementlarge}; GQA reduces the number of KV heads \citep{ainslie-etal-2023-gqa}.
Quantization reduces bits per element: KIVI uses per-channel keys and per-token values \citep{https://doi.org/10.13140/rg.2.2.28167.37282}, while KVQuant combines pre-RoPE key quantization with separate outlier storage \citep{hooper2025kvquant10millioncontext}.
Our focus is instead the number of coordinates stored per historical state, without changing the pretrained attention architecture.
We do not evaluate additional low-bit quantization or serving-system integration.

\paragraph{Low-rank KV representations.}
Low-rank methods exploit redundancy across features without discarding token positions.
Palu factorizes projection weights and caches low-dimensional intermediate states \citep{chang2024palu}; Eigen Attention performs attention in reduced spaces \citep{saxena-etal-2024-eigen}.
\shadowkv{} combines low-rank pre-RoPE keys, value offloading, and sparse retrieval \citep{sun2025shadowkvkvcacheshadows}.
More recent work learns bases against decoder-layer output error in StiefAttention \citep{benfenati2026don}, adapts head/block ranks in STAR-KV \citep{bhatnagar2026star}, and benchmarks data-dependent compressibility in KV-CoRE \citep{chen2026kv}.
These studies inform basis construction and rank choice.
Our question is how to incorporate newly decoded states while keeping historical coordinates consistent with an evolving basis.

\paragraph{Online low-rank updates.}
For newly generated states, \method{} maintains a shared historical factorization for each K/V path and updates its basis and old coordinates together.
The update uses Brand's classical incremental SVD \citep{Brand2002may,Brand2006may}, not a new matrix-decomposition primitive.
Its small core represents the stored history and incoming block before rank truncation (\Cref{prop:block-exact}).
Our contribution is this coordinated cache update with exact recent states, evaluated on long reasoning at matched persistent-memory budgets.

\suppressfloats[t]
\section{\method{}}\label{sec:method}

\begin{figure}[t]
  \centering
  \includegraphics[width=\textwidth]{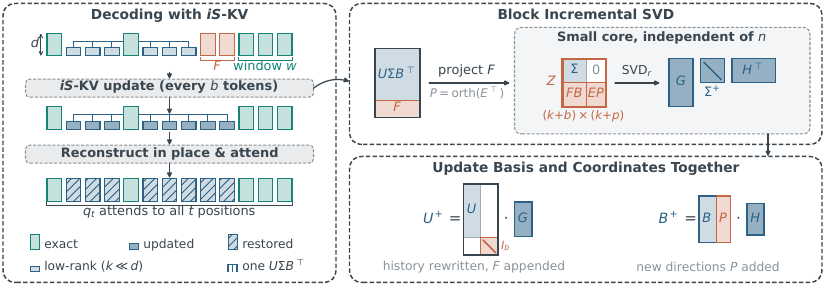}
  \caption{Overview of \method{}.
    The KV cache grows
    while \method{} keeps the recent window of \(w\) tokens exact (teal)
    and folds older tokens into a compact low-rank history (blue),
    one \(U\Sigma B^{\top}\) per block.
    Every \(b\) tokens, the pending block \(F\) (orange) is merged in
    by a Block Incremental SVD that updates basis and coordinates
    together and solves only a small core,
    independent of the history length \(n\).
    Each position can be reconstructed
    for query \(q_t\) attends to all \(t\) positions.
  }\label{fig:method-overview}
\end{figure}

\method{} compresses the growing KV cache while retaining a representation of every token position.
Prefill initializes exact states and a low-rank history; decoding jointly updates the basis and historical coefficients as unprotected rows are incorporated in blocks.
\Cref{fig:method-overview} illustrates the workflow, with prefill and decoding specified in \Cref{alg:prefill,alg:decode} in \Cref{app:algorithms}.

\subsection{Objective and Prefill Compression}
\label{sub:method:setup}\label{sub:method:prompt}

Consider a single Transformer layer processing a prompt of \(L\) tokens, with \(H_{\mathrm{kv}}\) KV heads of dimension \(D\).
Concatenating the KV heads along the feature dimension yields key and value matrices \(K,V\in\mathbb R^{L\times d}\), where \(d=H_{\mathrm{kv}}D\) and each row corresponds to one token.
Following \citet{sun2025shadowkvkvcacheshadows}, we represent keys before RoPE (pre-RoPE keys).
For \(X\in\{K,V\}\), we partition the prompt positions \(\{1,\ldots,L\}\) into an exact set \(\mathcal E_X\) and a compressed set \(\mathcal C_X\).
Let \(X_{\mathcal C_X}:=X[\mathcal C_X]\in\mathbb R^{n_X\times d}\) denote the original submatrix selected for compression, where \(n_X=|\mathcal C_X|\).
\method{} stores \(X[\mathcal E_X]\) unchanged at its original precision and represents \(X_{\mathcal C_X}\) by a low-rank approximation \(\widehat X_{\mathcal C_X}\) of the same dimensions.
We retain the original position indices for reconstruction.

Given this partition and a rank budget \(r_X\), \method{} chooses \(\widehat X_{\mathcal C_X}\) to minimize the squared reconstruction error in the Frobenius norm:
\begin{equation}
 {\min}_{\widehat X_{\mathcal C_X}\in\mathbb{R}^{n_X \times d}}
 \quad \norm{X_{\mathcal C_X}-\widehat X_{\mathcal C_X}}_F^2
 \quad \emph{st.}
 \quad \operatorname{rank}(\widehat X_{\mathcal C_X})\le r_X.
\label{eq:cache-objective}
\end{equation}
The truncated singular value decomposition (SVD), denoted by \((U_X,\Sigma_X,B_X)=\operatorname{SVD}_{r_X}(X_{\mathcal C_X})\), retains the leading \(k_X=\min(r_X,n_X,d)\) singular components.
Here, \(U_X\in\mathbb R^{n_X\times k_X}\) and \(B_X\in\mathbb R^{d\times k_X}\) contain the retained left and right singular vectors, with orthonormal columns; \(\Sigma_X\in\mathbb R^{k_X\times k_X}\) is diagonal with the retained singular values.
The cache stores \(U_X\), \(B_X\), and the diagonal of \(\Sigma_X\); each row of \(U_X\Sigma_X\) contains a token's coefficients in the shared basis \(B_X\).

We determine which rows remain exact during prefill, whose attention uses the original keys and values.
For keys, we retain a short local suffix of the prompt, including any incomplete chunk, and partition the remaining prefix into fixed-size contiguous chunks.
For chunk \(c\) and KV head \(h\), let \(k^{\mathrm{rot}}_{i,h}\in\mathbb R^D\) denote the post-RoPE key at position \(i\), and \(\bar k^{\mathrm{rot}}_{c,h}\) the mean key in that chunk.
Following the outlier selection criterion of \citet{sun2025shadowkvkvcacheshadows}, we score each chunk by
\(
 s_{c,h}=\min_{i\in c}\cos\!\parens{k^{\mathrm{rot}}_{i,h},\bar k^{\mathrm{rot}}_{c,h}}
\).
We select up to \(m\) lowest-scoring chunks per KV head as outlier chunks; the union of their token positions across heads, together with the local suffix, defines the protected positions \(\mathcal P_K\).
Selection uses post-RoPE keys, but each protected position retains its full pre-RoPE key row across all KV heads.

For values, we define the protected prompt positions \(\mathcal P_V\) as all prompt positions within \(\rho\) tokens of any position in \(\mathcal P_K\), including \(\mathcal P_K\) itself.
We initialize an exact recent window of size \(w\) with the last \(\min(w,L)\) prompt values and denote their positions by \(W_V\).
The initial exact sets are therefore \(\mathcal E_K=\mathcal P_K\) and \(\mathcal E_V=\mathcal P_V\cup W_V\).
We factorize the remaining rows with the truncated SVD above.
The protection sets \(\mathcal P_K,\mathcal P_V\) remain fixed during decoding, while the recent window advances with generation.
Protection may cover an entire short prompt, leaving an empty low-rank history.
The chunk size and protection parameters are specified in \Cref{app:impl}.

\begin{samepage}
\subsection{Block-Incremental Decoding}
\label{sub:method:update}\label{sub:method:overview}\label{sub:method:recent}

During decoding, new key and value rows remain exact in recent windows of size \(w\).
Unprotected rows leaving these windows remain exact; their positions enter FIFO queues \(\mathcal Q_K\) for keys and \(\mathcal Q_V\) for values.
For \(X\in\{K,V\}\), \(\mathcal Q_X\) denotes the corresponding queue.
Regular updates consume the oldest \(b\) pending rows once \(|\mathcal Q_X|\ge b\), retaining at most \(r_X\) singular components.
Before regular value updates begin, a one-time refit is performed: once \(|\mathcal Q_V|\ge r_V\), we append the oldest \(r_V\) pending rows to the reconstructed prompt value history and compute a truncated SVD with rank budget \(r_V\).
Thus, \(r_V\) controls both the value rank budget and the initial refit size.
\par
\end{samepage}

Each regular update represents the history and new block in a common updated basis.
Omitting the subscript \(X\), let \(A=U\Sigma B^\top\in\mathbb R^{n\times d}\) denote the maintained history with \(k\le r\) components and rank budget \(r\), and \(F\in\mathbb R^{b\times d}\) the next pending block.
We seek an approximation \(\widehat M\) of their concatenation \(M=[A;F]\):
\begin{equation}
 {\min}_{\widehat M\in\mathbb R^{(n+b)\times d}}
 \quad \norm{M-\widehat M}_F^2
 \quad \emph{st.}
 \quad \operatorname{rank}(\widehat M)\le r.
 \label{eq:local-objective}
\end{equation}
The target \(M\) contains the maintained approximation \(A\) and exact block \(F\); the original compressed rows are no longer stored.
An empty low-rank history is initialized by a truncated SVD of \(F\).
Otherwise, block incremental SVD \citep{Brand2006may} updates the factors without reconstructing \(A\).
To incorporate directions outside \(B\), we compute \(E=F-(FB)B^\top\) and construct a residual basis \(P\in\mathbb R^{d\times p}\) by QR and SVD (\Cref{app:residual-impl}).
We retain directions whose singular values exceed the numerical threshold, with \(p\le\min(b,d-k)\) and \(p=0\) when none are retained.
Queue occupancy determines update timing; residual thresholding determines basis expansion.
Expressing \(A\) and \(F\) in the augmented basis \([B\;P]\) yields
\begin{equation}
 M\approx\begin{bmatrix}U&0\\0&I_b\end{bmatrix}Z[B\;P]^\top,
 \qquad Z=\begin{bmatrix}\Sigma&0\\FB&FP\end{bmatrix},
 \label{eq:block-core}
\end{equation}
where \(I_b\) is the \(b\times b\) identity matrix.
This reduces the SVD to a \((k+b)\times(k+p)\) core, independent of the history length \(n\).
Computing \((G,\Sigma^+,H)=\operatorname{SVD}_r(Z)\), with \(Z\approx G\Sigma^+H^\top\), gives
\begin{equation}
 U^+=\begin{bmatrix}U&0\\0&I_b\end{bmatrix}G,
 \qquad B^+=[B\;P]H.
 \label{eq:sync-update}
\end{equation}
The updated factors reconstruct \(\widehat M^+=U^+\Sigma^+(B^+)^\top\).
The first \(n\) and last \(b\) rows of \(U^+\Sigma^+\) contain the coefficients of \(A\) and \(F\), respectively, in the updated basis \(B^+\).

In exact arithmetic, with orthonormal columns in \(U,B\) and a complete orthonormal residual basis \(P\), \eqref{eq:block-core} is exact and the core SVD solves \eqref{eq:local-objective}, with squared error \(\sum_{i>r}\sigma_i(Z)^2\).
Residual thresholding and finite precision introduce additional error.
We periodically reorthogonalize the factors using QR and SVD, preserving their product in exact arithmetic.
Numerical settings and proofs are given in \Cref{app:impl}.

\begin{samepage}
\subsection{Reconstruction and Storage Cost}\label{sub:method:complexity}

At each decoding step, \(\mathcal E_X\) and \(\mathcal C_X\) denote the current exact and compressed position sets, with \(\mathcal E_X\) containing protected, recent, and pending positions.
Let \(X_{\mathrm{exact}}[i]\) denote the stored exact row at position \(i\). The reconstructed cache \(\widetilde X\) is defined in token order by
\begin{equation}
 \widetilde X[i]=
 \begin{cases}
 (U_X\Sigma_XB_X^\top)[j], & i=\mathcal C_X[j],\\
 X_{\mathrm{exact}}[i], & i\in\mathcal E_X.
 \end{cases}
 \label{eq:cache-reconstruct}
\end{equation}
\end{samepage}
Reconstruction aligns keys and values by their original token positions despite their different compressed position sets.
After restoring head dimensions, we apply RoPE to keys at their original positions.
For a rotated query \(q_t^{\mathrm{rot}}\) associated with KV head \(h\), the attention output is
\begin{equation}
 o_t=\operatorname{softmax}\!\left(q_t^{\mathrm{rot}}(\widetilde K_h^{\mathrm{rot}})^\top/\sqrt D\right)\widetilde V_h.
 \label{eq:reconstructed-attention}
\end{equation}

Let \(n_X=|\mathcal C_X|\) and \(e_X=|\mathcal E_X|\) denote the numbers of compressed and exact rows.
The factors and exact rows use \(\beta_f\) and \(\beta_e\) bytes per scalar, respectively.
Excluding metadata and temporary dense reconstruction buffers, the per-layer persistent storage is
\begin{equation}\textstyle
 M_{\text{\method{}}}
 =\beta_f\sum_{X\in\{K,V\}}k_X(n_X+d+1)
 +\beta_e d\sum_{X\in\{K,V\}}e_X.
 \label{eq:physical-memory}
\end{equation}
At equal rank \(r<d\) and precision, each compressed key-value pair requires \(2r\) coefficients instead of \(2d\) scalars, plus shared factors.
Storage therefore grows as \(O(nr)\) with history length.
Updating historical coefficients costs \(O(nr(r+b))\) per block; reorthogonalization and reconstruction also depend on \(n\).

\section{Experiments}\label{sec:results}

We evaluate reasoning accuracy, key/value compression and basis adaptation, long-input retrieval, and generation time and memory.

\subsection{Long-CoT Reasoning}\label{sec:results:cot}

\paragraph{Models, data, and evaluation.}
Our main evaluation uses all 500 MATH-500 questions \citep{hendrycks2021measuring} with DeepSeek-R1-Distill-Llama-8B \citep{guo2025deepseek} and Qwen3-8B \citep{yang2025qwen3technicalreport}.
We report final-answer accuracy from one sampled generation per question, using temperature 0.6, top-\(p=0.95\), top-\(k=50\), and a total prompt-plus-generation limit of 16,384 tokens.
Prompts follow each model's chat template and request a boxed final answer; empty or non-finite outputs count as incorrect.
We compare with the original model and the token-eviction methods R-KV \citep{cai2026r} and SnapKV \citep{li2024snapkvllmknowslooking}.
The original model, R-KV, and \method{} use paired questions, prompts, question-specific seeds, and the same answer evaluator.
SnapKV results are retained from the existing baseline summaries; \Cref{tab:main-results} distinguishes their inferred correct counts from directly counted outcomes.

\paragraph{Compression settings and memory matching.}
For MATH-500, \method{} uses ranks \(r\in\{64,96,128\}\), a 128-token exact recent window on each path, and 64-row updates.
Persistent factors use FP16, while updates and reconstruction use FP32; further settings are given in \Cref{app:impl}.
We match baselines by persistent KV bytes, not by nominal rank or retained-token count.
The accounting includes protected prompt states, exact recent states, low-rank coefficients, singular values, bases, and baseline staging storage.
All footprints are evaluated at the original model's output lengths, so different stopping lengths do not change the reference used for matching.
These ratios exclude model weights and temporary reconstruction tensors and therefore do not measure peak GPU-memory savings.

\begin{table*}[t]
  \centering
  \caption{\textbf{MATH-500 accuracy at matched persistent-KV footprints.}
  The evaluation uses 500 questions and a 16K total-token limit.
  KV ratio and compression are measured against the original model at its output lengths, counting R-KV's
  128-token staging buffer and SnapKV's 128-token compression interval.
  Shading marks \method{}; $\Delta$ is in percentage points.
  SnapKV correct counts are inferred by rounding its reported accuracy times 500,
  not counted from per-question outputs; its reported percentages are not all
  compatible with single-sample, 500-question accuracy increments.}
  \label{tab:main-results}
  \resizebox{\textwidth}{!}{%
  \begin{tabular}{lcrrrrr}
    \toprule
    \thead{Method} & \thead{Setting} &
    \thead{KV ratio $\downarrow$} & \thead{Compression $\uparrow$} &
    \thead{Correct} & \thead{Accuracy $\uparrow$} &
    \thead{$\Delta$} \\
    \midrule
    \rowcolor{black!12}
    \multicolumn{7}{l}{\textit{DeepSeek-R1-Distill-Llama-8B}} \\
    Original & -- & 100.00\% & 1.000$\times$ & 418/500 & \textbf{83.6} & 0.0 \\
    R-KV & budget 748 & 24.63\% & 4.060$\times$ & 391/500 & 78.2 & $-5.4$ \\
    \snapkv{} & $B=636$ & 24.64\% & 4.059$\times$ & 351/500 & 70.1 & $-13.5$ \\
    \rowcolor{black!8}
    \method{} & $r=128$ & 24.63\% & 4.060$\times$ & 413/500 & \underline{82.6} & \underline{$-1.0$} \\
    \cmidrule(lr){1-7}
    R-KV & budget 562 & 20.78\% & 4.813$\times$ & 382/500 & 76.4 & $-7.2$ \\
    \snapkv{} & $B=517$ & 20.80\% & 4.807$\times$ & 337/500 & 67.4 & $-16.2$ \\
    \rowcolor{black!8}
    \method{} & $r=96$ & 20.79\% & 4.811$\times$ & 409/500 & \underline{81.8} & \underline{$-1.8$} \\
    \cmidrule(lr){1-7}
    R-KV & budget 409 & 16.89\% & 5.920$\times$ & 363/500 & 72.6 & $-11.0$ \\
    \snapkv{} & $B=396$ & 16.90\% & 5.917$\times$ & 322/500 & 64.3 & $-19.3$ \\
    \rowcolor{black!8}
    \method{} & $r=64$ & 16.90\% & 5.917$\times$ & 391/500 & \underline{78.2} & \underline{$-5.4$} \\
    \cmidrule(lr){1-7}
    \midrule
    \rowcolor{black!12}
    \multicolumn{7}{l}{\textit{Qwen3-8B}} \\
    Original & -- & 100.00\% & 1.000$\times$ & 472/500 & \textbf{94.4} & 0.0 \\
    R-KV & budget 797 & 21.35\% & 4.683$\times$ & 357/500 & 71.4 & $-23.0$ \\
    \snapkv{} & $B=794$ & 21.36\% & 4.683$\times$ & 407/500 & 81.30 & $-13.10$ \\
    \rowcolor{black!8}
    \method{} & $r=128$ & 21.36\% & 4.681$\times$ & 460/500 & \underline{92.0} & \underline{$-2.4$} \\
    \cmidrule(lr){1-7}
    R-KV & budget 638 & 17.73\% & 5.641$\times$ & 344/500 & 68.8 & $-25.6$ \\
    \snapkv{} & $B=637$ & 17.72\% & 5.644$\times$ & 366/500 & 73.17 & $-21.23$ \\
    \rowcolor{black!8}
    \method{} & $r=96$ & 17.72\% & 5.644$\times$ & 446/500 & \underline{89.2} & \underline{$-5.2$} \\
    \cmidrule(lr){1-7}
    R-KV & budget 479 & 14.06\% & 7.113$\times$ & 304/500 & 60.8 & $-33.6$ \\
    \snapkv{} & $B=479$ & 14.06\% & 7.113$\times$ & 325/500 & 65.04 & $-29.36$ \\
    \rowcolor{black!8}
    \method{} & $r=64$ & 14.05\% & 7.118$\times$ & 439/500 & \underline{87.8} & \underline{$-6.6$} \\
    \cmidrule(lr){1-7}
    \bottomrule
  \end{tabular}
  }
\end{table*}

\paragraph{MATH-500: higher accuracy at matched memory.}
\Cref{tab:main-results} shows that \method{} outperforms both eviction baselines at all six reported model--budget settings.
On DeepSeek-R1-Distill-Llama-8B, rank 128 reaches 82.6\% accuracy at \(4.06\times\) compression, compared with 83.6\% for the original model and 78.2\% for matched R-KV.
At \(4.81\times\) and \(5.92\times\) compression, the gains over R-KV are 5.4 and 5.6 percentage points, respectively.
The advantage also holds on Qwen3-8B: rank 96 achieves 89.2\% at \(5.64\times\) compression, compared with 68.8\% for R-KV and a reported 73.17\% for SnapKV.
Even at rank 64, Qwen3-8B retains 87.8\% accuracy with \(7.12\times\) compression, versus 94.4\% for the original model.

\paragraph{Accuracy degrades gradually as the rank decreases.}
Reducing rank from 128 to 64 lowers \method{} accuracy by 4.4 percentage points on DeepSeek-R1-Distill-Llama-8B and 4.2 points on Qwen3-8B.
Over the corresponding matched budgets, R-KV loses 5.6 and 10.6 points, respectively.
The advantage persists across ranks, but the growing gap to the original model shows that retaining positions alone cannot compensate for lost representation fidelity.

\paragraph{Comparison with OjaKV.}
We run the released OjaKV implementation \citep{zhu-etal-2026-ojakv} on DeepSeek-R1-Distill-Llama-8B and MATH-500 with the 16K total-token limit used above.
Its released checkpoint achieves 62.0\% accuracy at approximately \(1.30\times\) KV compression, compared with the original model's 83.6\% and \method{}'s 78.2--82.6\% at \(4.06\text{--}5.92\times\) compression (\Cref{tab:ojakv}).
We also reduce OjaKV's ranks to target the same compression range as \method{}.
All 118 generations collected across these three settings fall into repetitive loops and reach the length limit without EOS.
We stop these runs early and report the observed failures rather than an unmeasured accuracy.

\begin{table}[!htbp]
  \centering
  \caption{\textbf{OjaKV on DeepSeek-R1-Distill-Llama-8B and MATH-500.}
  The upper block uses all 500 questions at a 16K total-token limit.
  OjaKV uses its released checkpoint with error-budget parameter 0.075.
  The lower block reduces its aggregate rank to target higher compression; these runs were stopped early.
  ``--'' means not scored, not zero accuracy.}
  \label{tab:ojakv}
  \begin{tabular}{lrrl}
    \toprule
    Method / setting & Compression & Acc. (\%) & Observed results \\
    \midrule
    Original model & \(1.00\times\) & 83.6 & 418/500 correct \\
    OjaKV (released) & \(\approx 1.30\times\) & 62.0 & 310/500 correct \\
    \method{} (rank 128) & \(4.06\times\) & 82.6 & 413/500 correct \\
    \method{} (rank 96) & \(4.81\times\) & 81.8 & 409/500 correct \\
    \method{} (rank 64) & \(5.92\times\) & 78.2 & 391/500 correct \\
    \midrule
    \multicolumn{4}{l}{\textit{OjaKV with reduced ranks; compression values are targets}} \\
    OjaKV (rank 232) & \(4.06\times\) & -- & 30/30 loop to limit \\
    OjaKV (rank 192) & \(4.81\times\) & -- & 28/28 loop to limit \\
    OjaKV (rank 152) & \(5.92\times\) & -- & 60/60 loop to limit \\
    \bottomrule
  \end{tabular}
\end{table}

\paragraph{AIME: the advantage persists at a longer generation budget.}
We further evaluate Qwen3-8B on AIME 2024 and AIME 2025 with a 32,768-token total limit.
Each dataset contains 30 questions, with one sampled generation per question.
We match persistent KV memory at the original model's output lengths, including R-KV's 128-token staging buffer; the horizontal axis in \Cref{fig:qwen3-aime-matched-budget} gives the corresponding R-KV retained-token budget.
These runs store low-rank factors in BF16 to avoid the FP16 overflow observed at this length (\Cref{app:bf16}); both formats occupy two bytes per element.
As \Cref{fig:qwen3-aime-matched-budget} shows, \method{} exceeds R-KV at every tested budget on both datasets.
On AIME 2024, increasing rank from 64 to 192 raises accuracy from 43.3\% to 76.7\%, while matched R-KV rises from 13.3\% to 43.3\%.
At rank 192, \method{} answers 23 of 30 questions correctly, as does the original model.
AIME 2025 remains more difficult: the same rank achieves 56.7\%, above R-KV's 40.0\% but below the original model's 63.3\%.
With one sample per question, each answer changes accuracy by 3.3 percentage points; an equal total does not establish statistical equivalence.

\begin{figure}[t]
  \begin{minipage}[t]{0.485\textwidth}
    \centering
    \includegraphics[width=\linewidth]{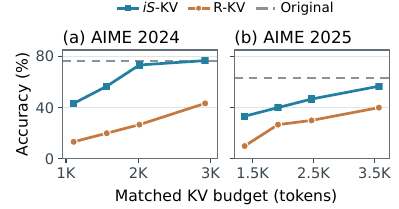}
    \caption{\textbf{AIME accuracy on Qwen3-8B.}
    Ranks 64, 96, 128, and 192 (left to right) use memory budgets matched to R-KV.
    Dashed lines show the original model.}
    \label{fig:qwen3-aime-matched-budget}
  \end{minipage}\hfill
  \begin{minipage}[t]{0.485\textwidth}
    \centering
    \includegraphics[width=\linewidth]{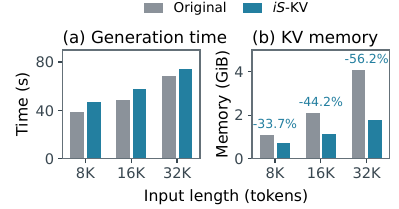}
    \caption{\textbf{Generation time and KV memory.}
    DeepSeek-R1-Distill-Llama-8B on an A40, rank 128: prefill plus 1,024 decode steps.
    Labels show KV savings over the original model.}
    \label{fig:latency}
  \end{minipage}
\end{figure}

\subsection{Ablation Study}\label{sec:results:ablation}

We ablate key/value compression and basis adaptation on a fixed prefix of 80 MATH-500 questions with DeepSeek-R1-Distill-Llama-8B, paired prompts and seeds, and one generation per question.
All compressed variants use rank 96 and the same exact-recent policy.
For the frozen-basis variant, both bases are initialized from the prompt and kept fixed during decoding.

\begin{table}[H]
  \centering
  \newcommand{\abcheck}{\textcolor{green!50!black}{\ensuremath{\checkmark}}}
  \newcommand{\abcross}{\textcolor{red!70!black}{\ensuremath{\times}}}
  \caption{\textbf{Component ablation on a fixed 80-question MATH subset.}
  Checks/crosses enable/disable each component; frozen bases remain at their prompt initialization.
  Compressed variants share rank 96 and the exact-recent policy, with one sample per question.
  Compression is the persistent-KV compression ratio.}
  \label{tab:kv-ablation}
  \resizebox{\textwidth}{!}{%
  \begin{tabular}{lcccrrr}
    \toprule
    \thead{Variant} & \thead{Low-rank K} & \thead{Low-rank V} &
    \thead{Online update} & \thead{Compression} & \thead{Correct} & \thead{Accuracy (\%)} \\
    \midrule
    Original & \abcross & \abcross & \abcross & 1.00$\times$ & 67/80 & 83.8 \\
    K-only & \abcheck & \abcross & \abcheck & 1.64$\times$ & \textbf{68/80} & \textbf{85.0} \\
    V-only & \abcross & \abcheck & \abcheck & 1.64$\times$ & 65/80 & 81.2 \\
    Frozen basis & \abcheck & \abcheck & \abcross & \textbf{4.68$\times$} & 51/80 & 63.7 \\
    \midrule
    \method{} & \abcheck & \abcheck & \abcheck & 4.61$\times$ & 65/80 & 81.2 \\
    \bottomrule
  \end{tabular}}
\end{table}

\paragraph{Key and value compression have different effects.}
In \Cref{tab:kv-ablation}, K-only compression answers 68/80 questions correctly, compared with 67/80 for the original model.
V-only compression and joint K+V compression each answer 65/80 correctly.
Value compression appears more sensitive on this subset, with no further aggregate loss from compressing keys.
Equal totals need not imply identical errors, nor does the one-question gain establish an improvement over the original model.
This comparison fixes per-path rank, not total memory.

\paragraph{Online basis adaptation matters.}
Keeping the bases fixed reduces the score from 65/80 to 51/80, a 17.5-percentage-point drop under the same rank and recent-window settings.
This supports adapting beyond the prompt-only subspace.
The answer-level ablation complements \Cref{fig:motivation}(c), which instead measures reconstruction drift from unsynchronized historical coordinates.

\subsection{Long Prefill Evaluation}\label{sec:results:ruler}

\paragraph{Setup.}
We evaluate 11 \ruler{} tasks \citep{hsieh2024rulerwhatsrealcontext} at 64K context on GLM-4-9B-1M, Llama-3-8B-Instruct-Gradient-1048k, and Llama-3.1-8B, using ten examples per task.
\method{} and the original model use identical inputs; we report task-specific percentage scores and their unweighted mean.
At initialization, \method{} retains selected prompt chunks and protected neighborhoods at their original precision and factorizes the remaining states, as detailed in \Cref{app:impl}.
With short outputs, this primarily tests prompt-history access under the two configurations in \Cref{tab:ruler-frozen}.

\begin{table}[!htbp]
  \centering
  \caption{\textbf{Prefill evaluation on 11 \ruler{} tasks at 64K context.}
  Scores are percentages, with ten examples per task; Avg.\ is the task mean.
  S1--S3 denote single-needle retrieval; MK, MQ, and MV denote multi-key,
  multi-query, and multi-value retrieval; VT, FWE, and QA denote variable
  tracking, frequent-word extraction, and question answering.
  Rank 96 uses 128 exact-recent rows and 16-row folds.
  Rank 160 uses single-row \(K\) updates with slack 8 and a 512-row exact-recent
  \(V\) window. These are different configurations, not a rank-only ablation.
  Comp. is compression ratio.}
  \label{tab:ruler-frozen}
  \resizebox{\textwidth}{!}{%
  \begin{tabular}{llrrrrrrrrrrrrr}
    \toprule
    \thead{Model} & \thead{Method} & \thead{S1} & \thead{S2} & \thead{S3} &
    \thead{MK1} & \thead{MK2} & \thead{MQ} & \thead{MV} &
    \thead{VT} & \thead{FWE} & \thead{QA1} & \thead{QA2} & \thead{Avg.} &
    \thead{Comp. \(\uparrow\)} \\
    \midrule
    \multirow{3}{*}{GLM-4-9B-1M}
      & Original & 100.0 & 100.0 & 100.0 & 100.0 & 100.0 & 100.0 & 98.8 &
        100.0 & 78.2 & 72.4 & 36.5 & 89.6 & 1.00$\times$ \\
      & \method{} (rank 160) & 100.0 & 100.0 & 100.0 & 100.0 & 100.0 & 100.0 & 85.7 &
        100.0 & 71.4 & 72.4 & 36.5 & 87.8 & 2.67$\times$ \\
      & \method{} (rank 96) & 100.0 & 100.0 & 100.0 & 100.0 & 95.2 & 92.7 & 82.0 &
        100.0 & 71.4 & 72.4 & 36.5 & 86.4 & 3.89$\times$ \\
    \cmidrule(lr){1-15}
    \multirow{3}{*}{\tbox[l]{Llama-3-8B- \\ Instruct-Gradient- \\ 1048k}}
      & Original & 100.0 & 100.0 & 100.0 & 100.0 & 100.0 & 94.6 & 93.6 &
        80.6 & 80.4 & 71.1 & 39.6 & 87.3 & 1.00$\times$ \\
      & \method{} (rank 160) & 100.0 & 100.0 & 100.0 & 100.0 & 100.0 & 88.1 & 92.1 &
        80.6 & 80.4 & 69.0 & 39.6 & 86.3 & 2.29$\times$ \\
      & \method{} (rank 96) & 100.0 & 100.0 & 100.0 & 100.0 & 100.0 & 62.0 & 83.8 &
        80.6 & 85.1 & 69.0 & 39.6 & 83.6 & 2.54$\times$ \\
    \cmidrule(lr){1-15}
    \multirow{3}{*}{Llama-3.1-8B}
      & Original & 100.0 & 100.0 & 100.0 & 100.0 & 61.3 & 100.0 & 100.0 &
        94.0 & 76.6 & 58.3 & 51.5 & 85.6 & 1.00$\times$ \\
      & \method{} (rank 160) & 100.0 & 100.0 & 100.0 & 80.7 & 62.3 & 79.7 & 68.8 &
        94.0 & 76.0 & 56.2 & 51.5 & 79.0 & 2.23$\times$ \\
      & \method{} (rank 96) & 100.0 & 100.0 & 85.7 & 69.8 & 62.3 & 46.6 & 41.6 &
        92.6 & 87.4 & 56.2 & 46.7 & 71.7 & 2.47$\times$ \\
    \bottomrule
  \end{tabular}%
  }
\end{table}

\paragraph{Access to long input histories.}
Rank 160 matches the original models' 100.0 on S1--S3.
Its means are 87.8, 86.3, and 79.0 on GLM-4-9B-1M, Llama-3-8B-Instruct-Gradient-1048k, and Llama-3.1-8B, versus 89.6, 87.3, and 85.6 for the original models.
Rank 96 retains 100.0 on S1--S3 for GLM-4-9B-1M and Llama-3-8B-Instruct-Gradient-1048k, with means of 86.4, 83.6, and 71.7 across the models.
Multi-query and multi-value retrieval show larger differences.
MATH-500 and AIME primarily compress generated states because their short prompts normally remain exact.
Here, the 64K prompts require compression before decoding, so the results measure retrieval from combined exact and low-rank prompt storage.

\subsection{Efficiency}\label{sec:results:latency}

\paragraph{Setup.}
We use DeepSeek-R1-Distill-Llama-8B on an NVIDIA A40 at rank 128 and batch size one.
Both methods receive identical synthetic 8K, 16K, and 32K inputs followed by exactly 1,024 greedy decode steps.
We time prefill plus decoding, excluding model loading, extension compilation, and cache inspection, and count KV memory after decoding.
Each method--length pair has one fresh-process run, executed serially without inference warmup.

\paragraph{Generation time.}
In \Cref{fig:latency}(a), \method{} takes 47.45, 58.32, and 74.70 s at 8K, 16K, and 32K, versus 38.83, 49.04, and 68.63 s for the original model.
The overhead decreases from 22.2\% to 18.9\% and 8.8\%, including updates and reconstruction.
These single-run measurements describe this workload, not a statistically stable latency difference.

\paragraph{Memory.}
In \Cref{fig:latency}(b), the original model requires 1.125, 2.125, and 4.125 GiB at 8K, 16K, and 32K; \method{} uses 0.746, 1.185, and 1.808 GiB.
Savings increase from 33.7\% to 44.2\% and 56.2\% as more history enters the low-rank representation.
This includes factors and exact protected, recent, and pending rows, preserving all positions without making memory constant in length.
It excludes indices, unused capacity, reconstructed caches, and workspace.
\FloatBarrier

\section{Conclusion}\label{sec:conclusion}

Long chain-of-thought reasoning grows the KV cache
with every generated token.
Eviction methods shrink it by deleting tokens,
but a token that looks unimportant now may matter later,
and deletion cannot be undone.
\method{} instead keeps every token position
and stores old states compactly,
while recent states stay exact.
This lowers the cost of each stored token,
though memory still grows with the length of the trace.
The difficulty is doing this online:
the compact representation adapts as new tokens arrive,
and the old states are stored relative to it.
If it changes while they stay fixed,
those states may not be reconstructed well.
\method{} updates both together with block-incremental SVD,
so every update preserves the stored history.
Our takeaway is that compressing a growing cache
is a problem of representing every token faithfully
beyond just choosing which tokens to keep.

\bibliographystyle{iclr2027_conference}
\bibliography{references}

\appendix
\section{Implementation Details and Reproducibility}\label{app:impl}

This appendix gives the settings behind the reported results.
\Cref{app:impl:method} describes the method settings,
\Cref{app:impl:eval} the evaluation protocols,
\Cref{app:fig1} the probes behind \Cref{fig:motivation},
\Cref{sec:theory} the properties of the block update,
and \Cref{app:numerical-update} further numerical diagnostics.

\subsection{Method Settings}\label{app:impl:method}

\paragraph{Main configuration.}
The key and value caches use exact recent windows of 128 rows, block size 64,
and residual tolerance \(\tau=10^{-6}\).
Each update keeps at most the given rank budget.
MATH-500 uses ranks 64, 96, and 128 on both models, with all other settings unchanged.
Reorthogonalization runs after every 256 rows processed by regular block updates;
these rows include initialization from an empty history but not the first value refit.
The persistent factors \(U_X,B_X\) and the diagonal entries of \(\Sigma_X\) are
stored in FP16; projection, residual QR, core SVD, reorthogonalization, and
reconstruction use FP32. Factors are cast back to the storage precision after each update.
Each right basis is stored explicitly, protected prompt keys are stored once,
and no additional \(\Sigma_XB_X^\top\) product is cached.
Optional features of our code that this paper does not use are disabled:
prompt sketches, persistent residual corrections, dynamic historical-token scoring,
and integer quantization.

\paragraph{Prompt protection.}
During prefill, we partition post-RoPE keys into eight-token chunks, excluding the
last four chunks and any final incomplete chunk as a local suffix. For each KV
head, we score a chunk by the minimum cosine similarity between its token keys
and their chunk mean, then select the 128 lowest-scoring chunks (or all
available chunks for a shorter prompt). We take the union of their token
positions across heads and the local suffix. These positions retain exact
pre-RoPE keys and are excluded from key factorization; all other prompt keys
are factorized. Values within 32 tokens of the protected key positions remain
protected throughout decoding. The latest 128 prompt values initialize the exact recent
window; values outside both sets are factorized. The 64K \ruler{} evaluation uses this same rule, so its
prompt is only partly represented by low-rank factors.

\paragraph{Cache layout.}
The key cache stores pre-RoPE representations. Protected prompt rows remain exact and are
excluded from the compact low-rank coefficient matrix. In the MATH protocol,
the chunk-selection rule normally covers the entire short prompt, which
therefore remains exact. Decoded keys enter an exact FIFO buffer; once it contains
$128+64$ rows, the oldest 64 are incorporated into the low-rank history.
The value cache retains protected prompt rows and a recent window of 128 rows exactly.
The value rank budget \(r_V\) also determines the number of rows used in the
initial value refit. Once \(r_V\) unprotected value rows have left the recent
window, we append them to the reconstructed prompt value history
and recompute a truncated SVD with rank budget \(r_V\). Later value updates
use 64-row blocks. Incomplete blocks remain uncompressed.
Recent and eligible unprotected rows share one exact buffer,
with eligibility determined by position. The separate \(W_X\) and \(\mathcal Q_X\)
in \Cref{alg:prefill,alg:decode} are logical index lists for these two subsets.

\paragraph{Residual basis construction.}\label{app:residual-impl}
The update schedule is determined by the number of eligible rows, independently
of the residual threshold. For a block \(F\), we first
compute \(E=F-(FB)B^\top\) and then replace it by
\(E-(EB)B^\top\) to reduce residual components along the current basis caused by finite precision.
If \(\|E\|_F\le\tau\) or \(k=d\), the residual basis is empty.
Otherwise, we compute a reduced QR factorization \(E^\top=Q_E R_E\)
and an SVD \(R_E=L_E D_E H_E^\top\).
Let \(\delta_1\ge\delta_2\ge\cdots\) denote the singular values on the diagonal of \(D_E\).
The retained residual basis is defined by
\[
 \tau_{\mathrm{eff}}=\max\!\left(\tau,\epsilon\max(b,d)\|F\|_F\right),
 \quad p=\min\!\left(d-k,\#\{i:\delta_i>\tau_{\mathrm{eff}}\}\right),
 \quad P=Q_E L_E[:,1:p],
\]
where \(\epsilon\) is machine epsilon at the computation precision and
\(\tau=10^{-6}\) by default.
The lower-right block of the core matrix is \(FP\), which equals \(EP\) when the residual is exact and
\(B^\top P=0\).

\subsection{Evaluation Protocols}\label{app:impl:eval}

\paragraph{MATH-500.}
We evaluate DeepSeek-R1-Distill-Llama-8B and Qwen3-8B in BF16 on all 500
questions. Prompts use each model's chat template and request a boxed
final answer. Each question receives one generation with temperature 0.6,
top-$p=0.95$, top-$k=50$, maximum total length 16,384, and a fixed
question-specific seed. The original model, \method{}, and R-KV use the same questions,
prompts, seeds, and answer evaluator, and every reported row covers all 500 questions.

\paragraph{Memory accounting.}
All cache numbers are persistent analytical bytes validated against the factor
layout in \eqref{eq:physical-memory}. Temporary reconstruction tensors and
model weights are excluded. Compression ratios in the main table use the original model's
output lengths. The implementation has no fused factor-aware attention kernel, and
no throughput or peak-HBM advantage is claimed.

\paragraph{R-KV matching.}
R-KV uses its released Hugging Face implementation, divide length 128, and a
128-token staging buffer. On DeepSeek-R1-Distill-Llama-8B, retained budgets 748, 562, and 409
match \method{} ranks 128, 96, and 64 at 4.060$\times$, 4.813$\times$, and
5.920$\times$ compression. Their paired accuracy advantages are 4.4, 5.4,
and 5.6 points, with question-bootstrap 95\% intervals $[1.6,7.2]$,
$[2.0,8.8]$, and $[2.2,9.0]$. On Qwen3-8B, budget 638 is matched using the original model's
mean total length of 4,317.452 tokens. Rank-96 \method{} and R-KV then give
5.6442$\times$ and 5.6410$\times$ compression; their relative mismatch is
0.056\%. R-KV staging rows are included in every calculation.

\paragraph{OjaKV evaluation.}
\Cref{tab:ojakv} reports our evaluation of the released OjaKV implementation on
DeepSeek-R1-Distill-Llama-8B and MATH-500.
The released checkpoint, configured with error-budget parameter 0.075,
retains approximately 77\% of the original model's persistent KV storage.
It answers 310/500 questions correctly; three empty outputs count as incorrect.
For higher compression, we change its aggregate ranks to 232, 192, and 152,
targeting \(4.06\times\), \(4.81\times\), and \(5.92\times\) compression.
These runs are stopped after 30, 28, and 60 generations, respectively:
every collected output repeats until the approximately 16K-token limit,
without emitting EOS. These outputs are not scored.

\paragraph{Component ablation.}
The key/value ablation and the frozen-basis ablation use a fixed prefix of 80
MATH-500 questions on DeepSeek-R1-Distill-Llama-8B, with the same prompts, seeds,
evaluator, and single-generation protocol as the main MATH-500 runs.
All compressed variants use rank 96, exact recent windows of 128 rows, block size 64,
and reorthogonalization every 256 folded rows. The frozen-basis variant initializes
both low-rank bases from the prompt, freezes them before the first decode fold,
and keeps the same exact recent windows afterwards.

\paragraph{RULER.}
The rank-96 configuration of the 64K \ruler{} evaluation uses exact recent windows of 128 rows
but an earlier update setting: folds contain 16 rows and
reorthogonalization runs every 32 folded rows. Every task has ten examples.
\Cref{tab:ruler-frozen} also lists the different update and recent-window settings
of the rank-160 configuration.

\paragraph{Factor precision at a 32K budget.}\label{app:bf16}
The persistent factors are FP16 in the MATH-500 runs. At a 32,768-token
generation budget on Qwen3-8B they overflow: entries exceed the FP16 maximum of
65,504, become $\pm\infty$, and the logits turn into NaN.
On the same question, we observed this with three versions of our code:
the version used for the main table, the latest version, and the latest version
with all speed optimizations turned off, at 13K, 19K, and 55K decoded tokens, respectively.
FP32 and BF16 factors both completed the same generation,
so the failure comes from the representation rather than from one implementation.
The original model and R-KV are unaffected because they store raw keys and values in BF16,
whose largest finite value is about $3.4\times10^{38}$, whereas \method{} stores derived
factors that are updated incrementally and grow with sequence length.

We therefore store factors in BF16 for all 32K experiments. BF16 occupies the
same two bytes as FP16, so the compression ratios are unchanged; FP32 would cut
Qwen3-8B rank-128 compression from 6.47$\times$ to about 3.8$\times$.
BF16 gives up three mantissa bits for a much larger range.
Because BF16 and FP16 factors do not give bitwise identical results, we report the 32K results
as a separate configuration and do not rerun MATH-500 in BF16.
In the one setting run in both formats, DeepSeek-R1-Distill-Llama-8B at rank 128 on AIME 2024,
accuracy is the same (56.7\%), but BF16 lengthens mean generations from 15,541 to 18,126 tokens
and doubles the share of generations that reach the length limit from 13\% to 27\%.
The lower precision therefore still changes when generation stops.

\subsection{Probe Settings}\label{app:fig1}

\paragraph{Settings of \Cref{fig:motivation}.}
Panel (a) uses the longest trace of DeepSeek-R1-Distill-Llama-8B and Qwen3-8B
on each of MATH-500, AIME 2024, and AIME 2025. At decode step 1024, tokens are
ranked by their average attention over the preceding 32 steps, and the bottom
75\% are marked. No token is removed, so this panel does not measure the
accuracy of eviction; \Cref{tab:main-results} does. A marked token counts as
strongly attended once its attention at a later step exceeds 25 times the
uniform value \(1/n\), where \(n\) is the number of visible tokens. Attention
is measured at layer 15. Panel (b) uses four 8K-token AIME 2024 decode traces
per model; singular values are divided by the largest, and medians are taken
over layers and then over traces. In panels (a) and (b), lines show two-model
means and bands show two-model ranges. Panel (c) uses the update-rule probe
below: three DeepSeek traces, layers 0, 15, and 31, rank 32, and 64-row
blocks. Drift is the relative error with respect to the represented history
before each update, not the original KV cache; it differs from the factor drift
of \eqref{eq:drift}. Dashed lines update the basis
and keep the old coordinates; solid lines update the basis and all old
coordinates jointly, which in exact arithmetic equals the \method{} update
(\Cref{prop:block-exact,prop:trunc}). Lines show means and bands span
layer-averaged traces. All bands are observed ranges, not confidence intervals.

\paragraph{Fidelity probes.}
Factor drift reconstructs the low-rank state and compares it with a fresh SVD
of the same represented history at the same rank. The update-rule probe
extracts query, key, and value features from MATH samples 138, 450, and 216 at layers 0, 15,
and 31, then compares a stale-coordinate basis update with a synchronized Brand
update using rank 32 and a 64-row block. End-to-end probes teacher-force the
original model's continuation through rank-96 and rank-64 caches and report centered
vocabulary-logit error, top-1 agreement, EOS-logit shift, and teacher-token
logit shift.

\subsection{Properties of the Block Update}\label{app:standard-props}\label{sec:theory}

We establish the reconstruction, truncation, and reorthogonalization
properties of one block update. The analysis concerns the concatenation of
the maintained low-rank history and the arriving block, as defined in
\Cref{eq:local-objective}.

\begin{proposition}[Exact block incorporation before truncation]
\label{prop:block-exact}
Assume exact arithmetic and \(A=U\Sigma B^\top\), with orthonormal columns in
\(U\) and \(B\). For a new block \(F\), let \(P\) have orthonormal columns
spanning the complete row space of \(E=F-(FB)B^\top\).
If the core \(Z\) in \Cref{eq:block-core} is factorized without truncation,
the update in \Cref{eq:sync-update} satisfies
\[
  U^+\Sigma^+(B^+)^\top
  = \begin{bmatrix}A\\F\end{bmatrix}.
\]
\end{proposition}

\begin{proof}
Since \(B^\top P=0\) and \(E=EPP^\top\), the arriving block satisfies
\(F=(FB)B^\top+(FP)P^\top\). Hence
\[
\begin{bmatrix}A\\F\end{bmatrix}
=
\begin{bmatrix}U&0\\0&I_b\end{bmatrix}
\begin{bmatrix}\Sigma&0\\FB&FP\end{bmatrix}
[B\;P]^\top.
\]
Substituting the untruncated SVD of \(Z\) yields \Cref{eq:sync-update}.
\end{proof}

\begin{proposition}[Optimal local truncation]\label{prop:trunc}
Under the assumptions of \Cref{prop:block-exact}, let \(M=[A;F]\) and let
\(\sigma_1(Z)\ge\cdots\ge\sigma_q(Z)\) be the singular values of the core.
Let \(M_r\) be the reconstruction obtained by retaining its leading
\(\min(r,q)\) singular components. Then
\[
  \|M-M_r\|_F
  =\left(\sum_{i=r+1}^{q}\sigma_i(Z)^2\right)^{1/2}
  =\min_{\substack{\widehat M\in\mathbb R^{(n+b)\times d}\\\operatorname{rank}(\widehat M)\le r}}\|M-\widehat M\|_F.
\]
\end{proposition}

\begin{proof}
The outer factors in \Cref{eq:block-core} have orthonormal columns, so \(M\)
and \(Z\) have the same nonzero singular values. The result follows from
the Eckart--Young--Mirsky theorem.
\end{proof}

\paragraph{Effect of residual thresholding.}
When \(P\) retains only part of the residual row space, define
\(R_{\!\perp}=F-(FB)B^\top-(FP)P^\top\).
Assuming exact arithmetic, orthonormal columns in \(U,B,P\), and
\(B^\top P=0\), the discarded component is orthogonal to the retained
feature subspace. The update therefore satisfies
\[
 \|M-\widehat M^{+}\|_F^2
 =\|R_{\!\perp}\|_F^2+\sum_{i>r}\sigma_i(Z)^2.
\]
This follows by decomposing the error into its orthogonal components inside
and outside \(\operatorname{span}([B\;P])\).
The discarded core singular values therefore characterize the full error
only when the residual basis spans the complete residual row space. Finite-precision factors introduce further
deviations from these exact-arithmetic identities.

\paragraph{Local optimality and factor drift.}
The optimality in Proposition~\ref{prop:trunc} is local to the maintained
history and the arriving block. Since the maintained history already
incorporates earlier approximations, successive updates need not recover
the rank-\(r\) SVD of the original uncompressed history.
For a nonzero reference matrix \(X\) used to evaluate reconstruction, let
\(\widehat X_{\mathrm{inc}}\) denote the incremental approximation and
\(\widehat X_{\mathrm{fresh}}\) the direct truncated-SVD approximation at the same rank.
Their difference in relative reconstruction error is
\begin{equation}\label{eq:drift}
  \drift =
  \frac{\|\widehat X_{\mathrm{inc}}-X\|_F}{\|X\|_F}
  -
  \frac{\|\widehat X_{\mathrm{fresh}}-X\|_F}{\|X\|_F}.
\end{equation}
This quantity, referred to as factor drift, measures excess relative reconstruction
error. Logit probes against the original model and free-generation
probes evaluate the effects of cache approximation on model outputs.

\begin{proposition}[Product-preserving reorthogonalization]\label{prop:drift}
In exact arithmetic, the periodic reorthogonalization used by \method{} leaves
\(U\Sigma B^\top\) unchanged.
\end{proposition}

\begin{proof}
Let \(U=Q_U R_U\) and \(B=Q_B R_B\) be thin QR factorizations, and let
\(R_U\Sigma R_B^\top=\widetilde U\widetilde\Sigma\widetilde B^\top\)
be an untruncated SVD. The factors \(Q_U\widetilde U\), \(\widetilde\Sigma\), and
\(Q_B\widetilde B\) have the same product \(U\Sigma B^\top\).
In finite precision, the factors are orthonormal up to computation roundoff;
casting them to storage precision introduces additional rounding error.
\end{proof}

\begin{proposition}[Block-update cost]\label{prop:cost}
For block size \(b\), width \(d\), rank budget \(r\), and \(n\) factorized rows,
projection and residual computation require \(O(bdr)\) work.
Residual basis construction factorizes a \(d\times b\) matrix, and the core matrix
has at most \(r+b\) rows and columns. Updating the historical coefficients costs
\(O(nr(r+b))\). Persistent factor storage is \(O((n+d)r)\).
Periodic reorthogonalization adds work that depends on the number \(n\) of historical rows.
\end{proposition}

\subsection{Numerical Diagnostics}\label{app:numerical-update}

\paragraph{Residual threshold.}
These probes test the numerical fidelity of the update, not the gap to the original KV cache
or generation quality. With an invalid, oversized residual threshold, value drift reaches 0.961
(step 1,024, layer 0) and 0.906 (step 2,048, layer 15); the stable update gives 0.0001 and
0.0109 at the same points.

\paragraph{Drift over a long trajectory.}
We measure the drift of \eqref{eq:drift} on one AIME 2024 trajectory that never emits
EOS (8,192 decoded rows, FP16 persistent factors, reorthogonalization
every 256 folded rows). \Cref{tab:drift-8k} gives the values at step 8,192.
The key path stays within 0.007 excess error and the value path within
0.016--0.021, so the numerical cost of the incremental update is far smaller
than truncation itself on keys, and values are the more constrained of the two.

\begin{table}[!htbp]
\centering
\caption{\textbf{Factor drift at step 8,192.}
Excess relative reconstruction error \eqref{eq:drift} on one AIME 2024 trajectory,
by path, rank, and layer.}
\label{tab:drift-8k}
\small
\begin{tabular}{lccc|ccc}
  \toprule
  & \multicolumn{3}{c}{\textbf{$r=96$}} & \multicolumn{3}{c}{\textbf{$r=128$}} \\
  \textbf{Path} & \textbf{L0} & \textbf{L15} & \textbf{L31}
  & \textbf{L0} & \textbf{L15} & \textbf{L31} \\
  \midrule
  key   & 0.0024 & 0.0071 & 0.0046 & 0.0016 & 0.0065 & 0.0043 \\
  value & 0.0207 & 0.0176 & 0.0168 & 0.0194 & 0.0185 & 0.0160 \\
  \bottomrule
\end{tabular}
\end{table}

\paragraph{Reorthogonalization.}
On the same states, reorthogonalization reduces the basis orthogonality error
$\|\mU^\top\mU-\mI\|_F$ from $9.5\times10^{-3}$ to $1.1\times10^{-3}$
(layer 15, $r{=}128$; $8$--$10\times$), while the median relative change of the
represented product across reorthogonalization events is $2.4\times10^{-4}$,
consistent with \Cref{prop:drift}.
Within 8K steps, reorthogonalization does not reduce reconstruction error;
it keeps the bases orthogonal, which guards against error build-up in longer generations.

\paragraph{Per-depth spectra.}\label{app:obs-spectra}
For pre-RoPE keys, the rank \(r_{99\%}\) that keeps 99\% of the spectral energy
is not stable across depth: beyond layer 3, its coefficient of variation is 0.19
on DeepSeek-R1-Distill-Llama-8B and 0.26 on Qwen3-8B, where it spans 122--338.
The energy kept at the ranks we use is stable: at $r{=}128$ the coefficient of
variation across depth is $\le0.023$ for both key paths and $\le0.055$ for values.
This stability, not a constant \(r_{99\%}\), is what supports a single rank setting for all layers.

\section{Algorithms}\label{app:algorithms}
The prefill initialization and per-step decoding procedures
summarized in \Cref{sec:method} are specified
in \Cref{alg:prefill,alg:decode}.
The accompanying paragraphs walk through their steps
and point to the definitions and equations they instantiate.

\paragraph{Prefill initialization (\Cref{alg:prefill}).}
\Cref{alg:prefill} realizes the construction
of \Cref{sub:method:prompt}.
It first evaluates exact causal attention
over the prompt, so the prefill outputs never depend
on the compressed representation.
It then selects the protected key positions \(\mathcal P_K\)
from the chunk outlier scores \(s_{c,h}\) of \Cref{sec:method},
taking the union of the selected chunks with the local suffix
and dilating them by \(\rho\) positions to form
the protected value positions \(\mathcal P_V\);
the last \(\min(w,L)\) prompt positions initialize
the exact recent windows \(W_K\) and \(W_V\).
For each of \(K\) and \(V\),
the exact set \(\mathcal E_X=\mathcal P_X\cup W_X\)
is stored verbatim and the remaining positions
\(\mathcal C_X=\mathcal I\setminus\mathcal E_X\) are factorized
by the truncated SVD \(\operatorname{SVD}_{r_X}\) of \Cref{sec:method},
retaining \(k_X=\min(r_X,|\mathcal C_X|,d)\) components.
The pending queues \(\mathcal Q_X\) and update counters \(c_X\) start empty,
and the flag \(f_V\) records that the one-time value refit
of \Cref{sub:method:update} is still pending.

\makeatletter
\providecommand{\theHALG@line}{}
\renewcommand{\theHALG@line}{\thealgorithm.\arabic{ALG@line}}
\makeatother

\begingroup
\algrenewcommand{\algorithmicrequire}{\textbf{Input:}}
\algrenewcommand{\algorithmicensure}{\textbf{Output:}}
\algrenewcommand{\algorithmicindent}{1.2em}
\algrenewcommand{\alglinenumber}[1]{\textcolor{black!55}{\scriptsize #1}}
\algtext*{EndFor}
\newcommand{\MethodParameters}{\item[\textbf{Parameters:}]}
\begin{algorithm}[H]
\caption{\method{}: prefill initialization at one layer}
\label{alg:prefill}
\fontsize{9}{10.5}\selectfont
\begin{algorithmic}[1]
\Require Prompt queries \(Q\) and keys \(K\) before RoPE, values \(V\), positions \(\mathcal I=\{1,\ldots,L\}\)
\MethodParameters Rank budgets \(r_K,r_V\), window size \(w\), protection settings from \Cref{sub:method:prompt}
\Ensure Prefill output \(O\); cache \(\mathcal S\) with factors, exact rows, position lists, ranks, counters, and flag \(f_V\)
\State \(O\gets\operatorname{CausalAttn}(\operatorname{RoPE}_{\mathcal I}(Q),\operatorname{RoPE}_{\mathcal I}(K),V)\) \Comment{Exact prefill attention}
\State \(\mathcal P_K\gets\) union of outlier-chunk positions across heads and the local suffix \Comment{Outlier scores, \Cref{sec:method}}
\State \(\mathcal P_V\gets\{i\in\mathcal I:\exists j\in\mathcal P_K,\ |i-j|\le\rho\}\)
\State \(W_K\gets\varnothing\); \(W_V\gets\) the last \(\min(w,L)\) prompt positions \Comment{Initialize recent windows}
\For{\(X\in\{K,V\}\)}
  \State \(\mathcal E_X\gets\mathcal P_X\cup W_X\); \(\mathcal C_X\gets\mathcal I\setminus\mathcal E_X\), ordered by position
  \State \(X_{\mathrm{exact}}[\mathcal E_X]\gets X[\mathcal E_X]\)
  \State \((U_X,\Sigma_X,B_X)\gets\operatorname{SVD}_{r_X}(X[\mathcal C_X])\)
  \State \(k_X\gets\min(r_X,|\mathcal C_X|,d)\); \(\mathcal Q_X\gets\varnothing\); \(c_X\gets0\)
\EndFor
\State \(f_V\gets\mathrm{true}\) \Comment{Initial value refit is pending}
\end{algorithmic}
\end{algorithm}
\endgroup

\paragraph{Decoding step (\Cref{alg:decode}).}
\Cref{alg:decode} maintains the cache
of \Cref{sub:method:update} during generation.
For each of the key and value streams it stores
the arriving row exactly, advances the recent window,
and enqueues the unprotected positions that leave it.
A fold is triggered once the queue holds a full block:
an empty history is initialized by a truncated SVD of the block,
the pending value refit appends the oldest \(r_V\) rows
to the reconstructed value history and recomputes
a rank-\(r_V\) truncated SVD, and otherwise
the \textsc{BlockUpdate} subprocedure applies the synchronized
block-incremental update of \eqref{eq:block-core} and \eqref{eq:sync-update}.
Queue occupancy sets the update timing, the tolerance
\(\tau\) controls the residual basis of \Cref{app:residual-impl},
and the factors are reorthogonalized every \(T\) folded rows,
which preserves their product by \Cref{prop:drift}.
The folded positions move from the queue into the compressed set \(\mathcal C_X\),
and the exact set becomes \(\mathcal E_X=\mathcal P_X\cup W_X\cup\mathcal Q_X\).
The remaining steps reconstruct the cache in token order
by \eqref{eq:cache-reconstruct} and form the attention output
of \eqref{eq:reconstructed-attention} (\Cref{sub:method:complexity}).

\begingroup
\algrenewcommand{\algorithmicrequire}{\textbf{Input:}}
\algrenewcommand{\algorithmicensure}{\textbf{Output:}}
\algrenewcommand{\algorithmicindent}{1.2em}
\algrenewcommand{\alglinenumber}[1]{\textcolor{black!55}{\scriptsize #1}}
\algtext*{EndIf}
\algtext*{EndFor}
\newcommand{\MethodParameters}{\item[\textbf{Parameters:}]}
\begin{algorithm}[H]
\caption{\method{}: one decoding step at one layer}
\label{alg:decode}
\fontsize{9}{10.5}\selectfont
\begin{algorithmic}[1]
\Require Cache \(\mathcal S\), pre-RoPE query \(q_t\) and key \(k_t\), value \(v_t\), position \(t\)
\MethodParameters Rank budgets \(r_K,r_V\), window \(w\), block size \(b\), tolerance \(\tau\), reorthogonalization interval \(T\)
\Ensure Attention output \(o_t\) and updated cache \(\mathcal S\)
\For{\((X,x_t)\in\{(K,k_t),(V,v_t)\}\)} \Comment{Update the cache state}
  \State Store \(x_t\) exactly at position \(t\) and advance \(W_X\)
  \State Enqueue positions leaving \(W_X\) and outside \(\mathcal P_X\) in \(\mathcal Q_X\)
  \State \(a\gets r_V\) if \(X=V\) and \(f_V\), otherwise \(b\)
  \If{\(|\mathcal Q_X|\ge a\)}
    \State \(J\gets\) the oldest \(a\) positions in \(\mathcal Q_X\); \(F\gets X_{\mathrm{exact}}[J]\)
    \If{\(X=V\) and \(f_V\)}
      \State \((U_V,\Sigma_V,B_V)\gets\operatorname{SVD}_{r_V}([U_V\Sigma_VB_V^\top;F])\); \(f_V\gets\mathrm{false}\)
    \Else
      \State \((U_X,\Sigma_X,B_X)\gets\Call{BlockUpdate}{U_X,\Sigma_X,B_X,F;r_X,\tau}\); \(c_X\gets c_X+b\)
      \If{\(T>0\) and \(c_X\ge T\)}
        \State Reorthogonalize \((U_X,\Sigma_X,B_X)\); \(c_X\gets c_X\bmod T\) \Comment{\Cref{prop:drift}}
      \EndIf
    \EndIf
    \State \(k_X\gets\) the number of columns of \(B_X\)
    \State Append \(J\) to \(\mathcal C_X\); remove \(J\) from \(\mathcal Q_X\) and delete the corresponding exact rows
  \EndIf
  \State \(\mathcal E_X\gets\mathcal P_X\cup W_X\cup\mathcal Q_X\)
\EndFor
\Statex \textit{\(\triangleright\) Reconstruct the cache and evaluate attention}
\State Reconstruct \(\widetilde K,\widetilde V\) in token order by \Cref{eq:cache-reconstruct}
\State Restore head dimensions; apply RoPE to keys and \(q_t\) at their original positions
\State Compute \(o_t\) by \Cref{eq:reconstructed-attention}
\end{algorithmic}
\vspace{2pt}
\noindent\textit{BlockUpdate:} for an empty low-rank history, return \(\operatorname{SVD}_{r_X}(F)\). Otherwise, construct \(P\) using tolerance \(\tau\), truncate the core SVD in \Cref{eq:block-core} to rank budget \(r_X\), and return the updated factors in \Cref{eq:sync-update}.
\end{algorithm}
\endgroup

\end{document}